\documentclass[letterpaper]{article}
\usepackage[preprint]{aaai2027}
\usepackage[hyphens]{url}
\usepackage{graphicx}
\usepackage{natbib}
\usepackage{caption}
\usepackage{amsmath}
\usepackage{amssymb}
\usepackage{booktabs}
\usepackage{amsthm}
\newtheorem{theorem}{Theorem}

\begin{document}

\title{Balanced Prompt Adaptation against Entropy-Induced Collapse\\
for Test-Time Binary Segmentation}

\author{
Zhengshan Wang,
Joshua Charles Webster-Ford,
Yifei Tian,\\
Xinxin Wang\textsuperscript{\rm 1}\corresponding,
Long Chen,
Weiping Ding
}
\affiliations{
\textsuperscript{\rm 1}Shenzhen University, Shenzhen, China\\
Corresponding author: xinxinwang1024@gmail.com\\
ORCID: \url{https://orcid.org/0009-0000-6065-7651}
}

\maketitle

\begin{abstract}
Entropy minimization is a standard objective for test-time adaptation (TTA), but it can fail in imbalanced binary segmentation. Unlike image classification, dense segmentation aggregates thousands of pixel predictions, allowing the larger predicted class to dominate the update, pull minority predictions toward itself, and produce a degenerate mask as predictions saturate and their entropy gradients vanish. We theoretically establish this collapse in a shared-shift model. This analysis motivates Balanced-Anchor Prompt Adaptation (BAPA), which combines two complementary modules. The Class-Balanced Anchors (CBA) module selects high-confidence anchors separately from each predicted class and gives foreground and background equal total loss weight, preventing the larger region from dominating the update. Dynamic Prompt Adaptation (DPA) refreshes these anchors after each prediction update and optimizes only text-side prompt residuals while keeping the vision--language encoders frozen. This prompt-only update refines the foreground--background decision boundary without altering the pretrained dense visual representation. Across experiments from four domains, BAPA achieves the highest mean Dice among the evaluated methods. Factorized ablations further validate the complementary roles of CBA and DPA, supporting balanced prompt adaptation as an effective alternative to entropy minimization for test-time binary segmentation.
\end{abstract}

%-------------------------------------------------------------------------
\section{Introduction}
%-------------------------------------------------------------------------

Test-time adaptation (TTA) seeks to recover model performance under distribution shift without labeled target data~\cite{wang2021tent,liang2023survey}. It is especially attractive for segmentation systems deployed across domains, because medical lesions, salient objects, pets, and open-vocabulary object categories often differ substantially from the data used to train the base model. Vision--language models (VLMs) further broaden this setting: a user can specify a foreground concept by text and obtain a binary mask without task-specific training~\cite{radford2021learning,noori2025mlmp}. The remaining question is how to adapt such source-free predictions safely at test time.

A common TTA strategy is to minimize prediction entropy~\cite{wang2021tent,noori2025mlmp,niu2023towards,niu2022efficient,shu2022test}, encouraging the model to make more confident predictions. This objective is attractive because it requires neither source data nor target annotations, but confidence alone cannot tell whether a prediction is correct. In image classification, entropy is applied to a small number of image-level outputs; in segmentation, the same objective aggregates thousands of pixel decisions. The resulting update is therefore shaped not only by uncertainty, but also by how many pixels are currently assigned to each class.

\begin{figure}[!t]
\centering
\includegraphics[width=\columnwidth,trim=63bp 129bp 264bp 89bp,clip]{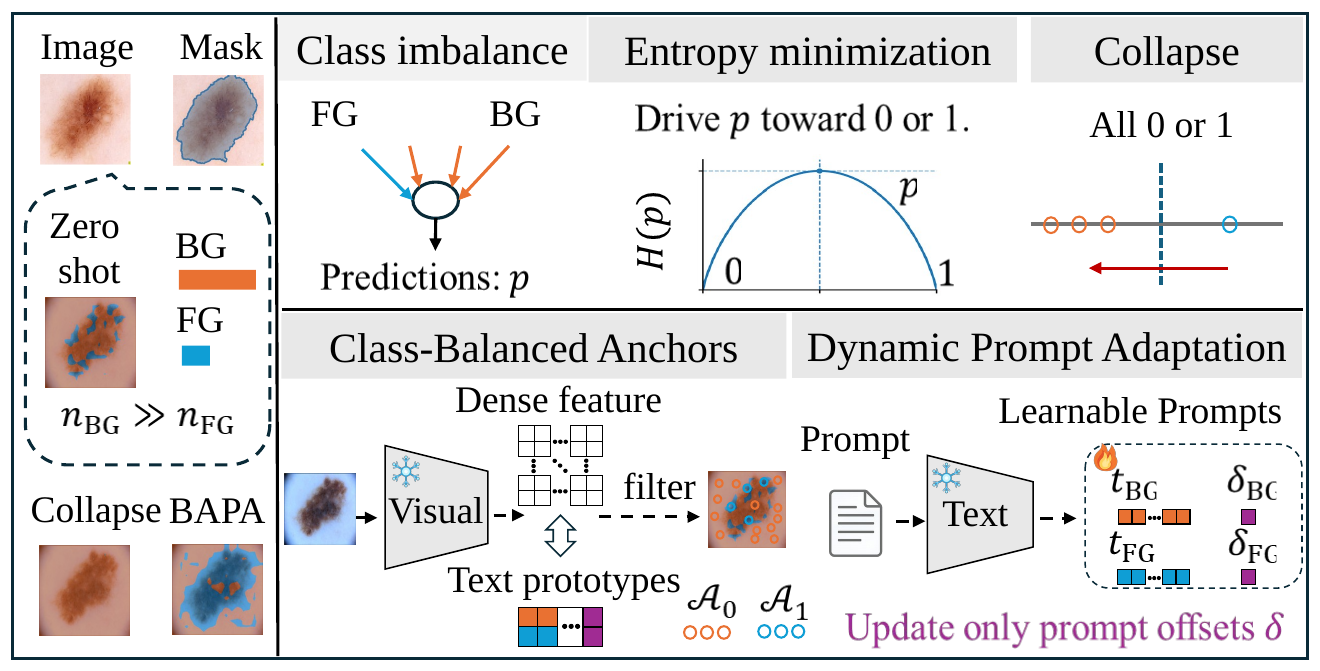}
\caption{Limitations of entropy-based TTA and balanced prompt adaptation.
In imbalanced binary segmentation, entropy-based adaptation can amplify the currently dominant predicted class and drive the mask toward a degenerate solution. BAPA addresses this failure through Class-Balanced Anchors (CBA) and Dynamic Prompt Adaptation (DPA). BG and FG denote background and foreground, respectively.}
\label{fig:landscape}
\end{figure}

This limitation becomes more severe in binary segmentation, where the foreground often occupies only a small part of the image. Entropy minimization pushes each pixel to become more confidently background (BG) or foreground (FG), whether its current prediction is correct or not; errors can therefore be sharpened rather than corrected. Because the same adaptation update is computed from all pixels, the majority prediction can also move minority pixels in the same direction, producing nearly all-background or all-foreground masks. This degeneration is referred to as \emph{entropy-induced collapse}.

This collapse is driven by two properties of the entropy objective. First, entropy minimization rewards confident background or foreground predictions, regardless of whether they are correct. Second, once an incorrect prediction becomes highly confident, the entropy gradient becomes weak, leaving little signal to undo the error. In imbalanced masks, the majority class can dominate the aggregate update and pull minority pixels in the same direction. This makes the failure different from ordinary overfitting to pseudo-labels: the objective itself can favor a degenerate but highly confident mask. Stabilization methods such as entropy filtering with SAR~\cite{niu2023towards}, Fisher regularization with EATA~\cite{niu2022efficient}, and weight averaging can alter the trajectory, but they still use entropy as the adaptation signal and do not provide class-balanced corrective targets.

The analysis leads to two complementary modules. The \emph{Class-Balanced Anchors} (CBA) module selects confident pixels separately from the current background and foreground predictions and gives the two classes equal total loss weight. This does not assume equal foreground and background areas; it only prevents the larger predicted region from exerting greater optimization influence. \emph{Dynamic Prompt Adaptation} (DPA) then reconstructs the anchors after each prediction update while optimizing only two text-side residuals and keeping the vision--language encoders frozen. The dynamic refresh allows improved intermediate masks to provide updated supervision, while the prompt-only update adjusts the foreground--background semantic boundary without altering dense visual features. Together, CBA and DPA form BAPA and turn adaptation from entropy-driven confidence sharpening into balanced correction of the current binary decision boundary.

Figure~\ref{fig:landscape} summarizes the motivation of the paper and illustrates how CBA and DPA jointly counter entropy-induced collapse.
This work makes three contributions:
\begin{enumerate}
    \item We identify entropy-induced collapse as a failure mode of entropy-based TTA in imbalanced binary segmentation and connect it to majority-class dominance in the adaptation signal.
    \item We derive Class-Balanced Anchors (CBA) from this analysis. CBA removes pixel-count dependence by giving foreground and background anchors equal total influence.
    \item We introduce Dynamic Prompt Adaptation (DPA), which refreshes the anchors while updating only text-side residuals over frozen vision--language encoders. Together, CBA and DPA form BAPA; dataset-wide experiments on 16,486 image--task instances across four domains validate both modules.
\end{enumerate}

\section{Related Work}
%-------------------------------------------------------------------------

\subsection{Entropy-based TTA}

TTA began from single-sample self-supervised updates~\cite{sun2020test}, entropy minimization~\cite{wang2021tent}, and augmentation-based marginal entropy minimization~\cite{zhang2022memo}. Later methods improve reliability through output-space adaptation~\cite{boudiaf2022parameter}, continual restoration and ensembling~\cite{wang2022cotta}, sharpness-aware filtering~\cite{niu2023towards}, Fisher regularization~\cite{niu2022efficient}, and weight averaging~\cite{osowiechi2024watt}. MLMP~\cite{noori2025mlmp} adapts LayerNorm parameters for vision--language segmentation, while TPT~\cite{shu2022test} applies entropy minimization to prompt embeddings over augmented views. Orthogonality-constrained prompt tuning further improves VLM calibration~\cite{sharifdeen2025otpt}. Recent benchmarks show that VLM-TTA gains depend strongly on protocol, update parameterization, and reliability metrics~\cite{sheng2025illusion,huang2026what}. These methods differ in parameterization and stabilization, but none directly supplies balanced foreground--background targets.

\subsection{Open-vocabulary segmentation}

CLIP~\cite{radford2021learning} and dense extensions~\cite{li2022lseg,rao2022denseclip,xu2022groupvit,liang2023open,xu2023odise,cho2024cat,noori2025mlmp} transfer language-aligned representations to pixel prediction. Their open-vocabulary capability supports source-free deployment, but zero-shot masks can be poorly calibrated under domain shift. BAPA leaves the visual representation fixed and adapts only the binary text decision boundary.

Recent work narrows the gap between classification TTA and dense prediction. Seg-TTO~\cite{desilva2025segtto} jointly optimizes segmentation-specific visual attributes and multiple textual embeddings on top of CAT-Seg or CLIP-DINOiser. MLMP~\cite{noori2025mlmp} is the closest architecture-compatible method because it supports single-image adaptation with dense NACLIP predictions; its complete official adaptation procedure is evaluated as a separate baseline. We discuss Seg-TTO as related segmentation TTA, but exclude it from the controlled comparison because it relies on a different segmentation architecture and checkpoints.

\subsection{Recent VLM and medical TTA}

Recent VLM-TTA methods also explore cache-based dynamic adapters~\cite{karmanov2024efficient}, distributional adapters~\cite{han2024dota}, transductive objectives~\cite{zanella2024transclip}, prompt-free test-time augmentation~\cite{zanella2024mta,farina2024zero}, statistical anchoring~\cite{zanella2025realistic}, contrastive objectives~\cite{lafon2025cliptta}, Bayesian class-prior adaptation~\cite{zhou2025bayesian}, noisy/open-set adaptation~\cite{cao2025noisy}, black-box prompt search~\cite{meng2025blackbox}, debiased prompt optimization~\cite{song2026doubly}, and view-filtered or robustness-oriented prompt tuning~\cite{choi2026dual,kim2026sstpt,zhu2026rita}. Medical and dense prediction TTA further use image normalization, shape priors, episodic feature adaptation, or low-rank VLM updates~\cite{karani2021testtime,bateson2022shape,valanarasu2022onthefly,noori2026histopathc}. Many of these methods assume image-level classification, augmented-view batches, persistent state, task-specific priors, or different segmentation backbones, whereas this work studies episodic single-image dense adaptation.

\begin{figure*}[t]
\centering
\includegraphics[width=\textwidth,trim=55bp 133bp 77bp 54bp,clip]{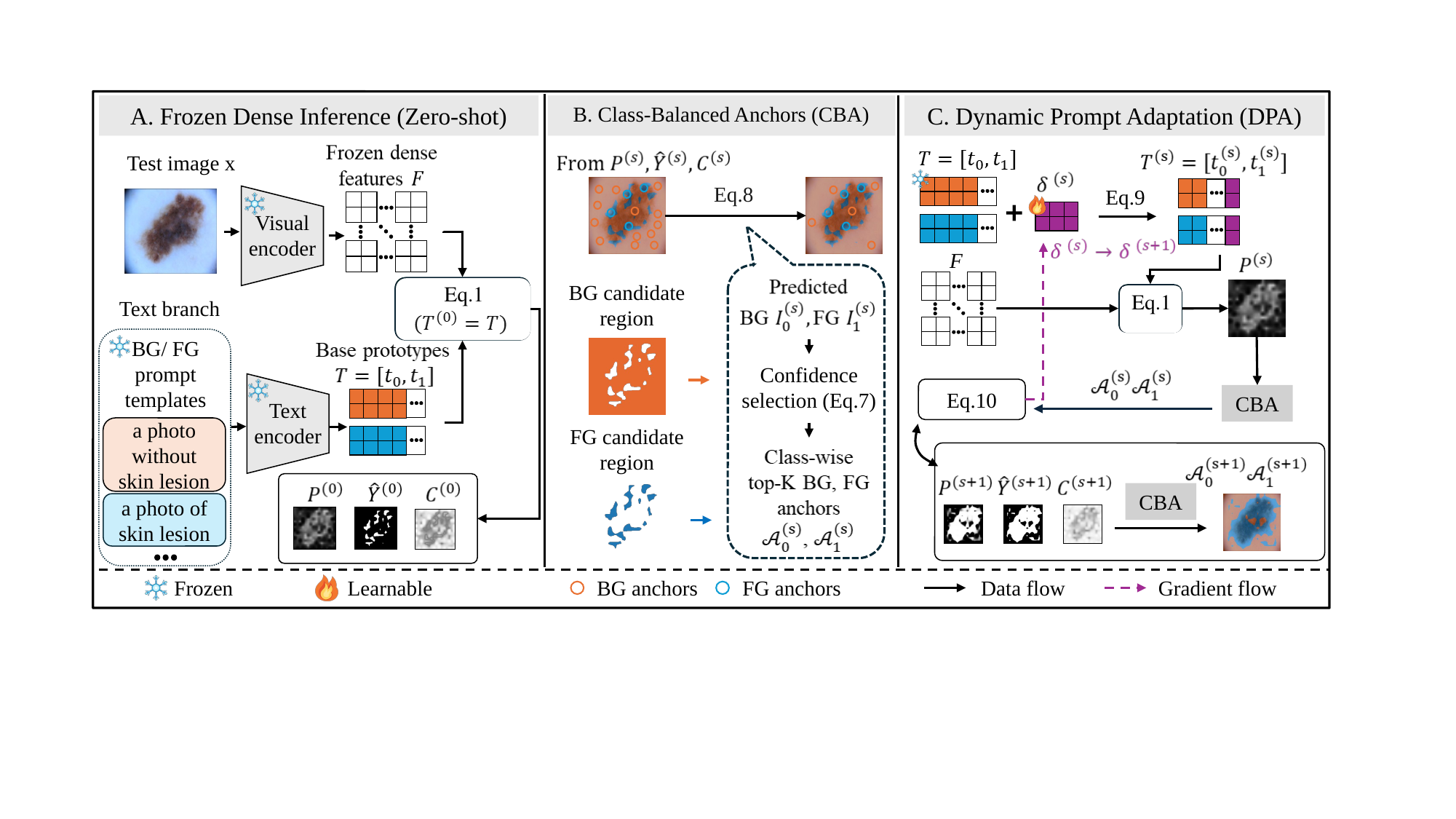}
\caption{BAPA framework. A frozen dense vision--language model (VLM) produces the zero-shot prediction. The Class-Balanced Anchors (CBA) module constructs foreground and background anchors through classwise top-$K$ confidence selection and equalizes their total loss contributions. Dynamic Prompt Adaptation (DPA) optimizes prompt residuals and refreshes the prediction, pseudo-labels, confidence, and anchors after every step, while both encoders and the dense visual representation remain fixed. BG and FG denote background and foreground, respectively.}
\label{fig:method}
\end{figure*}

\subsection{Pseudo-label adaptation}

Pseudo-labeling~\cite{lee2013pseudo}, self-training~\cite{xie2020self}, and confidence selection~\cite{sohn2020fixmatch,zou2018unsupervised,zou2019confidence} use model predictions as supervision, but iterative updates can amplify early errors through confirmation bias. This risk is acute in single-image binary segmentation, where pixel imbalance lets one predicted class dominate the update. CBA prevents this dominance through classwise selection and loss normalization, while DPA refreshes the resulting anchors as the prediction evolves.

%-------------------------------------------------------------------------
\section{Method}
%-------------------------------------------------------------------------

\subsection{Problem Definition}

We consider episodic test-time adaptation for binary segmentation, where a
model receives one unlabeled test image $\mathbf{x}$ and is reset before the
next image--task instance. Class-specific prompt templates are encoded,
averaged, and L2-normalized to obtain the background and foreground text
prototypes $\mathbf{t}_0$ and $\mathbf{t}_1$. Stacking them gives
$\mathbf{T}=[\mathbf{t}_0,\mathbf{t}_1]\in\mathbb{R}^{2\times d}$. The frozen
dense vision--language model produces
\begin{equation}
\mathbf{L}
=
\frac{f_v(\mathbf{x})\mathbf{T}^{\top}}{\tau},
\label{eq:zero_shot_logits}
\end{equation}
where $f_v(\mathbf{x})\in\mathbb{R}^{HW\times d}$ is the dense visual feature
map and $\tau=\exp(-\ell_{\mathrm{scale}})>0$ is the inverse of the frozen
checkpoint logit scale. The logits define per-pixel probabilities
$\mathbf{P}_{hw}=\operatorname{softmax}(\mathbf{L}_{hw})$, hard predictions
$\hat{Y}_{hw}=\arg\max_c P_{hw,c}$, and confidence scores
$C_{hw}=\max_c P_{hw,c}$. No evaluation labels or source samples are available
during adaptation. The objective is to improve the binary mask by updating a
restricted set of parameters while preserving the pretrained dense visual
representation.

\subsection{Overall Framework}

Figure~\ref{fig:method} summarizes BAPA as two complementary modules. CBA prevents the larger predicted region from dominating the loss by selecting confident foreground and background anchors separately and equalizing their total contributions. DPA reconstructs those anchors after each update and changes only the text prompts, leaving the dense visual features fixed.

\subsection{Analysis of Entropy-Induced Collapse}

\noindent\textbf{Entropy geometry.} Consider a binary segmentation model producing per-pixel logits $z_{hw}$ and foreground probability $p_{hw}=\sigma(z_{hw})$ via the sigmoid function. The standard TTA objective minimizes per-pixel entropy:

\begin{equation}
H(p) = -p\log p - (1-p)\log(1-p)
\label{eq:entropy}
\end{equation}

$H(p)$ has a unique local maximum at $p=0.5$ ($H(0.5)=\log 2$) and two global minima at $p=0$ and $p=1$ ($H(0)=H(1)=0$). The gradient with respect to the logit is:

\begin{equation}
\frac{dH}{dz} = \sigma(z)(1-\sigma(z))\log\frac{1-\sigma(z)}{\sigma(z)}
\label{eq:ent_grad}
\end{equation}

For $|z|\to\infty$, $\sigma(z)(1-\sigma(z))\sim e^{-|z|}$ and $|\log\frac{1-\sigma(z)}{\sigma(z)}|\sim |z|$. Hence, $|dH/dz|=O(|z|e^{-|z|})$. Entropy minimization supplies progressively less signal as a prediction becomes confident, including when it is incorrect. Once an erroneous prediction becomes highly confident, a finite number of adaptation steps may provide too little gradient to correct it.
%  This is an optimization argument, not a claim that finite logits are mathematically irreversible.

\noindent\textbf{Majority-dominated updates.} Let $\theta$ denote the adapted parameters and $z_i(\theta)$ the logit at pixel $i$. For the mean-entropy objective $\mathcal{L}_{\mathrm{ent}}=n^{-1}\sum_i H(\sigma(z_i))$, the parameter gradient is $\nabla_\theta \mathcal{L}_{\mathrm{ent}}=n^{-1}\sum_i (dH/dz_i)\nabla_\theta z_i$. This sum explains why pixel counts matter. If most pixels are currently predicted as background, their entropy gradients can dominate the update, especially when many pixels respond similarly to the adapted parameters. The resulting update can make background predictions more confident and move uncertain foreground pixels toward background. Class imbalance does not by itself guarantee collapse, because pixel-wise sensitivities also matter, but it provides a concrete and measurable source of majority-class bias.

\begin{figure}[t]
\centering
\includegraphics[width=0.8\columnwidth]{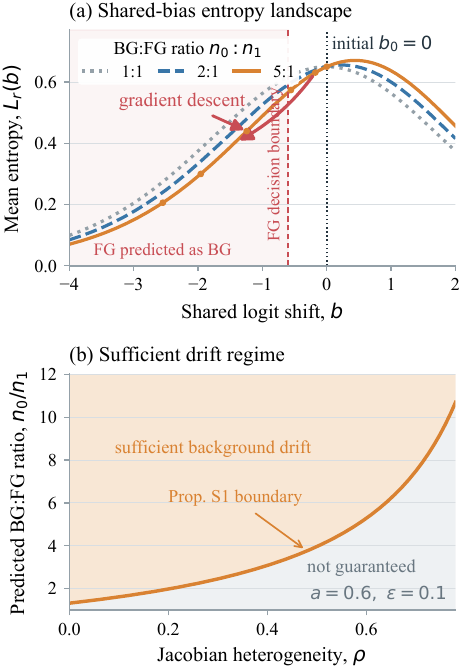}
\caption{Majority-dominated entropy updates and local drift condition.
(a) Increasing $n_0/n_1$ tilts the entropy landscape toward negative shared shifts; $b=-a$ is the foreground decision boundary. (b) For $a=0.6$ and $\epsilon=0.1$, the shaded upper region satisfies Supplementary Proposition~S1.}
\label{fig:drift-region}
\end{figure}

We first study this effect in the simplest setting where the adapted update moves all pixel logits in a common direction. Suppose $n_0$ pixels are initially predicted as background with logit $-a$, and $n_1$ pixels are initially predicted as foreground with logit $+a$, where $a>0$ measures the initial confidence margin and $n=n_0+n_1$. A scalar parameter $b$, initialized at $b_0=0$, shifts every logit by the same amount: $z_i(b)\in\{-a+b,\,a+b\}$. This shared shift is a simplified model of a prompt or normalization update that moves the foreground--background decision boundary. For the background-to-foreground count ratio $r=n_0/n_1$, Figure~\ref{fig:drift-region}(a) visualizes
\begin{equation}
L_r(b)=\frac{rH(\sigma(-a+b))+H(\sigma(a+b))}{r+1}.
\end{equation}
Define $f(x)=x\sigma(x)(1-\sigma(x))$ for $x\geq0$, and let $x_\star\approx1.54$ be the positive solution of $x\tanh(x/2)=1$, so that $f$ is increasing on $[0,x_\star]$.

\begin{theorem}[Majority-induced collapse under a shared update]
\label{thm:collapse}
Assume $n_0>n_1$, $b_0=0$, and $0<a\leq x_\star/2$. For the mean entropy
$L(b)=n^{-1}\sum_{i=1}^n H(\sigma(z_i(b)))$, gradient descent
$b_{t+1}=b_t-\eta L'(b_t)$ with any $\eta>0$ satisfies: (i) $b_t$ decreases; (ii) it reaches $b_t<-a$ in finitely many steps; and (iii) $b_t\to-\infty$, so every foreground probability converges to zero. If $n_0=n_1$, then $L'(0)=0$; if $n_1>n_0$, the symmetric all-foreground result holds.
\end{theorem}

\textit{Proof sketch.} For $b\in[-a,0]$, $L'(b)=[n_0 f(a-b)-n_1 f(a+b)]/n>0$ by monotonicity of $f$ and $n_0>n_1$. The derivative remains positive for $b<-a$, so $b_t\to-\infty$. The complete proof is provided in Supplementary Theorem~S1. The theorem captures a sufficient mechanism: class imbalance destabilizes the symmetric stationary point, and a shared parameter direction propagates the majority decision to every pixel.

The common-shift assumption can be relaxed locally. For a unit parameter direction $v$, let $j_i=\langle\nabla_\theta z_i(\theta_0),v\rangle$ denote the projected logit Jacobian of pixel $i$. Let $\rho$ measure the heterogeneity of these projected responses, with smaller $\rho$ indicating stronger alignment, and let $\epsilon$ bound the local logit deviation from $\pm a$. Supplementary Proposition~S1 shows that, when the projected Jacobians are sufficiently aligned, a large enough predicted background-to-foreground ratio $n_0/n_1$ makes the entropy-gradient step drift toward background. This is a sufficient local condition, but it gives a testable prediction: real text-residual entropy updates should contain a common foreground--background drift component, which we measure below.

Figure~\ref{fig:drift-region} visualizes the shared-shift result and its local extension. Panel (a) shows that a background majority tilts the mean-entropy landscape toward negative shared shifts, whereas balanced predictions keep the landscape symmetric around $b=0$. Panel (b) gives the sufficient local drift condition: higher Jacobian heterogeneity $\rho$ requires stronger class imbalance to certify drift. The shaded region is sufficient rather than necessary. Empirically, ISIC, VOC, and DUTS often exhibit background-majority zero-shot masks, while OxPet gives the symmetric foreground-majority case.

The observed results follow the same pattern. Entropy-based prompt tuning fails severely on ISIC, reaching 0.006 Dice, and LayerNorm entropy adaptation also reduces zero-shot Dice from 0.486 to 0.427. On less imbalanced VOC masks, the same adaptation improves zero-shot performance, while OxPet shows that class-balanced supervision is also useful under foreground-majority imbalance. Entropy-based stabilizers may change the trajectory, but they retain confidence seeking as the supervision signal and do not directly balance foreground and background contributions.

\noindent\textbf{Class-balanced correction.} The preceding analysis suggests replacing confidence-only entropy with an anchor-based objective whose foreground and background terms have equal total influence. For a pseudo-label anchor $y$ held fixed during one update, cross-entropy (CE) gives

\begin{equation}
\mathcal{L}_{\text{CE}}(p,y) = -y\log p - (1-y)\log(1-p),\quad
\frac{d\mathcal{L}_{\text{CE}}}{dz} = p - y
\label{eq:ce}
\end{equation}

Therefore, CE remains corrective when a confident prediction disagrees with its anchor. Class balancing then removes the pixel-count bias by averaging CE within each predicted class before combining the classes. In the shared-shift model, if the selected anchors are fixed locally and correct, the resulting objective is
\begin{equation}
L_{\mathrm{bal}}(b)=\tfrac12\operatorname{softplus}(-a+b)
+\tfrac12\operatorname{softplus}(-a-b)
\label{eq:balanced_ce}
\end{equation}
where $\operatorname{softplus}(x)=\log(1+e^x)$. The two terms are the background and foreground anchor losses, each weighted by $1/2$. The unique minimizer is $b^\star=0$, so the shared shift is no longer biased by the relative number of background and foreground pixels. Supplementary Theorem~S2 gives the convexity and gradient-descent convergence proof. This local result motivates CBA, while DPA reapplies the same class-balanced construction after each prediction update.

\subsection{Class-Balanced Anchor Construction}

At adaptation step $s$, CBA constructs anchors from the current prediction. Let $\hat{Y}_{hw}^{(s)}=\arg\max_c P_{hw,c}^{(s)}$ and $C_{hw}^{(s)}=\max_c P_{hw,c}^{(s)}$ denote the hard predicted class and confidence at pixel $(h,w)$. For each predicted class $c\in\{0,1\}$, define $\mathcal{I}_c^{(s)}=\{(h,w):\hat{Y}_{hw}^{(s)}=c\}$. CBA keeps only the top-$K$ most confident pixels within each predicted class. Specifically, if $\mathcal{I}_c^{(s)}$ is nonempty, $q_c^{(K,s)}$ is the confidence threshold that selects the top-$K$ fraction of pixels in $\mathcal{I}_c^{(s)}$; if no pixel is predicted as class $c$, then $\mathcal{A}_c^{(s)}=\emptyset$. The classwise anchor sets are

\begin{equation}
\begin{aligned}
\mathcal{A}_c^{(s)} &= \left\{(h,w)\in\mathcal{I}_c^{(s)} : C_{hw}^{(s)} \geq q_c^{(K,s)}\right\},\\
\mathcal{A}^{(s)} &= \mathcal{A}_0^{(s)}\cup\mathcal{A}_1^{(s)} .
\end{aligned}
\label{eq:anchors}
\end{equation}

The anchor cardinalities may differ because the predicted regions can have different sizes, and DPA updates the anchors after every adaptation step. Here, $\mathbf{P}_{hw}^{(s)}=[P_{hw,0}^{(s)},P_{hw,1}^{(s)}]$ denotes the current background--foreground probability vector at pixel $(h,w)$, obtained by applying softmax to the step-$s$ logits. Accordingly, $\mathcal{L}_{\mathrm{CE}}(\mathbf{P}_{hw}^{(s)},c)=-\log P_{hw,c}^{(s)}$. To remove the residual count imbalance, CBA averages the anchor loss within each class and then weights the two classwise means equally:

\begin{equation}
\mathcal{L}_{\mathrm{anchor}}^{(s)}
= \sum_{\substack{c\in\{0,1\}:\\|\mathcal{A}_c^{(s)}|>0}}\frac{1}{2|\mathcal{A}_c^{(s)}|}
\sum_{(h,w)\in\mathcal{A}_c^{(s)}}
\mathcal{L}_{\mathrm{CE}}(\mathbf{P}_{hw}^{(s)},c).
\label{eq:balanced_anchor_loss}
\end{equation}

Consequently, whenever both predicted classes are present, background and foreground each contribute total weight $1/2$, irrespective of $|\mathcal{A}_0^{(s)}|/|\mathcal{A}_1^{(s)}|$; an absent class contributes zero for that update. Classwise selection preserves confident evidence from both predicted classes, whereas Eq.~\ref{eq:balanced_anchor_loss} removes their pixel-count dependence in nondegenerate updates. We use $K=20\%$ throughout.

\subsection{Dynamic Prompt Adaptation}

A learnable residual $\boldsymbol{\delta}_c^{(s)}\in\mathbb{R}^{d}$ is added to the frozen text prototype $\mathbf{t}_c$ of each class $c\in\{0,1\}$ and initialized as $\boldsymbol{\delta}_c^{(0)}=\mathbf{0}$. The resulting prototype is then L2-normalized:

\begin{align}
\mathbf{t}_c^{(s)}
&= \frac{\mathbf{t}_c+\boldsymbol{\delta}_c^{(s)}}
{\|\mathbf{t}_c+\boldsymbol{\delta}_c^{(s)}\|_2},
\qquad c\in\{0,1\}, \label{eq:prompt}\\
\mathcal{L}^{(s)} &= \mathcal{L}_{\mathrm{anchor}}^{(s)} + \frac{\lambda}{2}\|\boldsymbol{\delta}^{(s)}\|_2^2 . \label{eq:loss}
\end{align}
Stacking $\mathbf{t}_0^{(s)}$ and $\mathbf{t}_1^{(s)}$ as rows gives $\mathbf{T}^{(s)}\in\mathbb{R}^{2\times d}$. Thus, Eq.~\ref{eq:prompt} normalizes each class prototype independently, rather than normalizing the two-class matrix jointly. Here, $\boldsymbol{\delta}^{(s)}$ stacks the two class residuals, and $\|\boldsymbol{\delta}^{(s)}\|_2^2=\sum_{c\in\{0,1\}}\sum_{j=1}^{d}(\delta_{c,j}^{(s)})^2$ measures their total squared magnitude. We set $\lambda=10^{-2}$ and implement this penalty through Adam weight decay, without adding it a second time to the optimized loss. The factor $1/2$ makes its conceptual gradient $\lambda\boldsymbol{\delta}^{(s)}$.

\noindent\textbf{Dynamic adaptation cycle.} At step $s$, $\boldsymbol{\delta}^{(s)}$ defines $\mathbf{T}^{(s)}$ and hence the current prediction $\mathbf{P}^{(s)}$. DPA recomputes $\hat{\mathbf{Y}}^{(s)}$ and $\mathbf{C}^{(s)}=[C_{hw}^{(s)}]_{hw}$ from this prediction, applies CBA to reconstruct $\mathcal{A}_c^{(s)}$, treats the resulting pseudo-labels and anchor indices as stop-gradient supervision, and updates only the prompt residual:
\begin{equation}
\boldsymbol{\delta}^{(s+1)}
\leftarrow \operatorname{Adam}\!\left(
\boldsymbol{\delta}^{(s)},
\nabla_{\boldsymbol{\delta}}\mathcal{L}^{(s)}
\right).
\label{eq:prompt_update}
\end{equation}
The updated residual changes the foreground--background similarity boundary while leaving the dense visual representation unchanged. A new prediction then updates the supervision for step $s+1$, yielding the cycle $\boldsymbol{\delta}^{(s)}\!\rightarrow\mathbf{P}^{(s)}\!\rightarrow\mathcal{A}^{(s)}\!\rightarrow\mathcal{L}^{(s)}\!\rightarrow\boldsymbol{\delta}^{(s+1)}$. Adam~\cite{kingma2014adam} performs $S=20$ updates with learning rate $10^{-3}$, $\beta=(0.9,0.999)$, and weight decay $10^{-2}$.

\subsection{Design rationale.} CBA removes region-size scaling from the supervision. DPA updates that supervision as the prediction evolves, but restricts optimization to the binary text boundary so that the pretrained dense representation remains unchanged.

%-------------------------------------------------------------------------
\section{Experiments}
%-------------------------------------------------------------------------

\subsection{Setup}

\noindent\textbf{Model and datasets.} All dataset-wide results use NACLIP ViT-L/14~\cite{hajimiri2025naclip} through the implementation distributed with MLMP~\cite{noori2025mlmp}, without offline fine-tuning. We evaluate 16,486 image--task instances from the ISIC 2017 training split~\cite{codella2018skin}, PASCAL VOC 2012~\cite{everingham2010pascal}, DUTS-TE~\cite{wang2017learning}, and Oxford-IIIT Pet~\cite{parkhi2012cats} at $224^2$ resolution. VOC uses all 20 foreground classes and reports both the class macro-average and the pooled instance-level mean.

\begin{table}[t]\centering
{\small
\setlength{\tabcolsep}{1.4pt}
\begin{tabular}{@{}lccccc@{}}
\toprule
\textbf{Method} & \textbf{ISIC} & \textbf{VOC} & \textbf{DUTS} & \textbf{OxPet} & \textbf{Mean} \\
\midrule
ZS~\citeyearpar{hajimiri2025naclip}      & 0.486 & 0.532 & 0.200 & 0.610 & 0.457 \\
PL~\citeyearpar{lee2013pseudo}           & 0.246 & 0.450 & 0.189 & 0.484 & 0.342 \\
TENT~\citeyearpar{wang2021tent}          & 0.427 & 0.542 & 0.184 & 0.627 & 0.445 \\
EATA~\citeyearpar{niu2022efficient}      & 0.466 & 0.536 & 0.194 & 0.614 & 0.453 \\
TPT~\citeyearpar{shu2022test}            & 0.006 & 0.253 & 0.087 & 0.700 & 0.261 \\
SAR~\citeyearpar{niu2023towards}         & 0.485 & 0.532 & 0.200 & 0.610 & 0.457 \\
WATT~\citeyearpar{osowiechi2024watt}     & 0.295 & 0.332 & \textbf{0.340} & 0.529 & 0.374 \\
CLIPArTT~\citeyearpar{hakim2024clipartt} & 0.287 & 0.337 & 0.283 & 0.516 & 0.356 \\
MLMP~\citeyearpar{noori2025mlmp}         & 0.489 & 0.493 & 0.241 & 0.723 & 0.486 \\
O-TPT~\citeyearpar{sharifdeen2025otpt}   & 0.006 & 0.304 & 0.087 & 0.708 & 0.276 \\
BAPA (Ours)                       & \textbf{0.587} & \textbf{0.578} & 0.296 & \textbf{0.783} & \textbf{0.561} \\
\bottomrule
\end{tabular}}
\caption{Dataset-wide Dice $\uparrow$. Parenthesized years cite each method; zero-shot inference (ZS) and the pseudo-label baseline (PL) precede the TTA methods. VOC pools all 2,077 eligible image--class instances; Supplementary Table~S5 reports its macro-average. Mean averages the four unrounded dataset scores.}
\label{tab:main}
\end{table}

\noindent\textbf{Protocol and baselines.} All methods use the same evaluation manifests, reset for each image--task instance, and use no evaluation labels for model selection. We compare ZS, PL~\cite{lee2013pseudo}, TENT~\cite{wang2021tent}, MLMP~\cite{noori2025mlmp}, TPT~\cite{shu2022test}, O-TPT~\cite{sharifdeen2025otpt}, SAR~\cite{niu2023towards}, EATA~\cite{niu2022efficient}, CLIPArTT~\cite{hakim2024clipartt}, WATT~\cite{osowiechi2024watt}, and BAPA under a unified dense single-image protocol. MLMP, CLIPArTT, and WATT use the dense open-vocabulary semantic segmentation (OVSS) adaptation classes distributed with the MLMP implementation. TENT and SAR call the official update routines through a dense-logit wrapper; EATA, TPT, and O-TPT use dense analogues of their official objectives because their original entry points assume image-level samples. BAPA uses one setting across datasets: $K=20\%$, $S=20$, learning rate $10^{-3}$, and weight decay $10^{-2}$. The Supplementary section \emph{Experimental Protocol and Reproducibility} specifies the manifests, implementation, and baseline ports.

\subsection{Main Results}

\begin{figure}[t]
\centering
\includegraphics[width=\columnwidth]{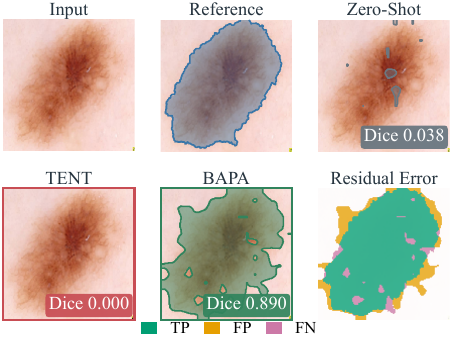}
\caption{Compact qualitative result. Panels show the input, reference mask, zero-shot prediction, TENT, BAPA, and BAPA residual errors for one representative ISIC lesion. Green, orange, and purple denote true positives, false positives, and false negatives in the residual-error panel.}
\label{fig:qual}
\end{figure}
\noindent\textbf{Dataset-wide comparison.}
Table~\ref{tab:main} uses every eligible VOC class for every method and reports one main Dice score per dataset. BAPA improves VOC pooled Dice from 0.532 to 0.578 and obtains the highest mean Dice across the four datasets. Across ISIC, VOC, and OxPet, BAPA also achieves the best individual dataset score; WATT is strongest on DUTS. Supplementary Tables~S3--S5 report the Oxford-Pet per-breed breakdowns and the complete VOC macro and per-class results; they show that BAPA improves every Oxford-Pet breed, raises VOC macro Dice from 0.517 to 0.575, and improves 14 of the 20 VOC classes.

\noindent\textbf{Qualitative results.}
Figure~\ref{fig:qual} complements the dataset-wide statistics with a single enlarged ISIC case study. TENT removes the lesion prediction, whereas BAPA recovers a localized foreground structure with substantially higher Dice. Supplementary Figure~S1 provides the full qualitative panel with additional ISIC and VOC examples.

\noindent\textbf{Paired inference.} Each BAPA result is paired with predictions from the comparison method on the same image--task instance. Relative to ZS, the 95\% paired-bootstrap intervals are above zero on every dataset. We also compare BAPA with the strongest non-BAPA method in each Table~\ref{tab:main} dataset column. For $\Delta=\text{Dice}_{\mathrm{BAPA}}-\text{Dice}_{\mathrm{baseline}}$, the intervals are ISIC versus MLMP $[0.088,0.109]$, VOC versus TENT $[0.028,0.043]$, DUTS versus WATT $[-0.052,-0.037]$, and OxPet versus MLMP $[0.058,0.062]$. All four differences remain significant after Holm correction over all 40 BAPA--comparator tests ($p_{\mathrm{Holm}}=0.004$). Thus, BAPA's highest four-dataset mean reflects significant gains on three datasets together with a significant deficit to WATT on DUTS, rather than uniform dominance. Resampling uses source images as independent units and clusters all VOC class tasks from the same image; the complete protocol and ZS comparisons are reported in the Supplementary Material.

\noindent\textbf{Mechanism diagnostic.} We decompose the entropy signal by defining $g_c$ as the gradient of mean entropy over pixels currently predicted as class $c$, with $c=0$ for background and $c=1$ for foreground. Among zero-shot masks containing both classes, the class-count-weighted log-ratio $\log_{10}(n_0\|g_0\|_2/(n_1\|g_1\|_2))$ averages 1.24 on ISIC, equivalent to a $17.21\times$ geometric background/foreground ratio; the ratios are 3.71 on VOC, 2.71 on DUTS, and 1.42 on OxPet. Its descriptive Pearson correlation with the binary zero-shot collapsed-mask indicator, defined using the 1\% foreground-fraction criterion in the Supplementary Material, is $r=0.72$, $0.45$, $0.28$, and $-0.29$, respectively. The negative OxPet value reflects the direction of the background-over-foreground ratio in a foreground-majority regime. The two gradients are strongly opposed on all datasets (mean cosine $-0.87$ to $-0.94$). All-BG and all-FG masks are excluded only where the ratio is undefined and remain in collapse-rate calculations.

This imbalance could still arise from unrelated pixelwise effects rather than a shared update direction. We therefore take one entropy-gradient step in the text-residual space $\boldsymbol{\delta}\in\mathbb{R}^{2\times d}$ and measure the resulting foreground-minus-background logit change $\Delta z_{hw}$ at each pixel. A negative $\Delta z_{hw}$ means that the pixel is pushed toward background. If the local theory captures the real update, many pixels should move together, so most of the $\Delta z$ energy should lie in a constant-pixel component. Table~\ref{tab:prompt_drift_main} shows exactly this pattern: the per-instance common-mode energy fraction averages 88.7--99.2\% across datasets, with the strongest drift on ISIC. Thus, the entropy gradient in the actual prompt space contains the shared foreground--background drift predicted by the local theory.

\begin{table}[t]\centering
{\small
\setlength{\tabcolsep}{4pt}
\begin{tabular}{@{}lccc@{}}
\toprule
\textbf{Dataset} & \textbf{$E_{\mathrm{com}}$} & \textbf{$\bar{\Delta z}$} & \textbf{FG$\to$BG} \\
\midrule
ISIC  & 99.2\% & $-0.121$ & 99.6\% \\
VOC   & 93.7\% & $-0.077$ & 82.7\% \\
DUTS  & 95.2\% & $-0.041$ & 66.4\% \\
OxPet & 88.7\% & $-0.032$ & 65.4\% \\
\bottomrule
\end{tabular}}
\caption{Prompt-space drift diagnostic. Entries are dataset means of per-instance quantities. $E_{\mathrm{com}}$ is the fraction of logit-change energy in the constant-pixel direction; $\bar{\Delta z}$ is the mean per-pixel logit change. These two metrics use every instance in each complete manifest. FG$\to$BG is averaged over instances with a nonempty foreground prediction.}
\label{tab:prompt_drift_main}
\end{table}

This diagnostic complements the correlation analysis by evaluating the entropy-induced direction in the actual prompt parameterization. The high $E_{\mathrm{com}}$ values show that the step contains a coherent foreground--background logit shift rather than only unrelated pixelwise changes. Its negative mean on all four datasets indicates a background-directed shared component, largest on ISIC, where the gradient-count imbalance is also strongest. These observations connect the sufficient local theory to real text-residual updates.

\subsection{Factorized Ablation}

\begin{table}[t]\centering
{\small
\setlength{\tabcolsep}{1.7pt}
\begin{tabular}{@{}cccccccc@{}}
\toprule
\textbf{Bal.} & \textbf{Policy} & \textbf{Params} & \textbf{ISIC} & \textbf{VOC} & \textbf{DUTS} & \textbf{OxPet} & \textbf{Mean} \\
\midrule
\checkmark & F & P  & 0.586 & \textbf{0.589} & 0.246 & 0.780 & 0.550 \\
--         & F & P  & 0.214 & 0.446 & 0.172 & 0.477 & 0.328 \\
\checkmark & D & P  & \textbf{0.587} & 0.578 & \textbf{0.296} & \textbf{0.783} & \textbf{0.561} \\
--         & D & P  & 0.003 & 0.247 & 0.084 & 0.476 & 0.203 \\
\checkmark & F & LN & 0.559 & 0.564 & 0.231 & 0.668 & 0.505 \\
--         & F & LN & 0.395 & 0.526 & 0.169 & 0.583 & 0.418 \\
\checkmark & D & LN & 0.563 & 0.565 & 0.234 & 0.669 & 0.508 \\
--         & D & LN & 0.392 & 0.525 & 0.164 & 0.583 & 0.416 \\
\bottomrule
\end{tabular}}
\caption{Dataset-wide $2\times2\times2$ factorized ablation (Dice $\uparrow$). Bal.=CBA; F=fixed zero-shot labels and anchor locations; D=dynamic anchor updates; P=prompt residuals; LN=LayerNorm. BAPA corresponds to the balanced D+P row. VOC pools all 2,077 eligible instances; Mean averages the four unrounded dataset scores.}
\label{tab:ablation}
\end{table}

Table~\ref{tab:ablation} factorizes CBA from the two choices that define DPA: dynamic anchor updates and prompt-residual optimization. With both DPA choices fixed, adding CBA changes Dice from 0.003 to 0.587 on ISIC, 0.247 to 0.578 on VOC, 0.084 to 0.296 on DUTS, and 0.476 to 0.783 on OxPet. Within the CBA rows, enabling dynamic refresh instead of fixed anchors changes Dice by +0.001, -0.011, +0.050, and +0.003. Prompt residuals also outperform the matched LayerNorm updates in every class-balanced setting. These comparisons separately support the balancing role of CBA and the update design of DPA.

\subsection{Robustness and Scope}

The complete BAPA configuration, combining CBA with DPA, has the highest Dice on DUTS and OxPet and is within 0.002 of the best ablation row on ISIC and within 0.011 on VOC. Sensitivity analysis shows small Dice variation across weight decay and 10--30 update steps, while the aggressive learning rate $5\times10^{-3}$ gives the lowest Dice in every dataset column.

\noindent\textbf{Scope and limitations.} The theory establishes collapse and local drift along aligned update directions, and the prompt-space diagnostic measures this component in real text-residual updates. Supplementary Theorem~S2 covers anchor-stable regions; the coupled dynamic regime is evaluated empirically by the factorized ablation. The evaluation uses one backbone family and segmentation-adapted versions of several baselines; additional architectures and native multi-class segmentation remain future work.

\noindent\textbf{Implications.} The central finding is that entropy minimization can be structurally misaligned with imbalanced binary segmentation. CBA provides the main stabilization by preventing either predicted class from dominating the anchor loss. On this balanced objective, DPA can refresh supervision from improved intermediate predictions while restricting the update to text prompts and preserving dense visual features. BAPA combines these two roles in a single adaptation procedure.

%-------------------------------------------------------------------------
\section{Conclusion}
%-------------------------------------------------------------------------

Entropy minimization can amplify the majority prediction in imbalanced binary segmentation and drive a mask toward a degenerate state. The local analysis and prompt-space diagnostic connect this failure to a shared foreground--background drift. BAPA addresses the two resulting requirements through CBA, which equalizes foreground and background supervision, and DPA, which refreshes that supervision while adapting only text prompts over frozen encoders. Across four datasets, BAPA improves zero-shot Dice in every case, attains the highest mean Dice, and leads on three datasets. These results show that dense test-time adaptation benefits from combining class-balanced supervision with restricted dynamic updates.

\small
\bibliography{references}

\end{document}